\documentclass{article}
\usepackage{amsmath,graphicx,mlspconf}
\usepackage[caption=false]{subfig}
\def\Pscr{\mathcal{P}} 
\def\Wscr{\mathcal{W}} 
\def\Zscr{\mathcal{Z}}
\def\Lh{\widehat{L}}
\def\EUBO{\mathrm{EUBO}}
\def\ERM{\mathrm{ERM}}
\def\MAP{\mathrm{MAP}}
\def\Ebb{\mathbb{E}}
\def\KL{\mathrm{KL}}
\def\unl{\mathrm{un-l}}
\def\Ph{\widehat{P}}
\def \AVU{\mathrm{AVU}}

\usepackage{amsmath, amssymb, bbm, xspace}
\usepackage{epsfig}
\usepackage{longtable}
\usepackage{color}
\usepackage{mathrsfs}
\usepackage{comment}

\usepackage{courier}

\newtheorem{theorem}{Theorem}[section]

\newtheorem{lemma}{Lemma}[section]

\newtheorem{proposition}[theorem]{Proposition}

\newtheorem{corollary}[theorem]{Corollary}

\newtheorem{definition}{Definition}[section]

\def\bkE{{\rm I\kern-.17em E}}
\def\bk1{{\rm 1\kern-.17em l}}
\def\bkD{{\rm I\kern-.17em D}}
\def\bkR{{\rm I\kern-.17em R}}
\def\bkP{{\rm I\kern-.17em P}}

\def\bkZ{{\bf{Z}}}

\def\bkE{{\rm I\kern-.17em E}}
\def\bk1{{\rm 1\kern-.17em l}}
\def\bkD{{\rm I\kern-.17em D}}
\def\bkR{{\rm I\kern-.17em R}}
\def\bkP{{\rm I\kern-.17em P}}

\makeatletter
\newcommand{\pushright}[1]{\ifmeasuring@#1\else\omit\hfill$\displaystyle#1$\fi\ignorespaces}
\newcommand{\pushleft}[1]{\ifmeasuring@#1\else\omit$\displaystyle#1$\hfill\fi\ignorespaces}
\makeatother

\def\bkZ{{\bf{Z}}}
\def\b12{(\beta_1,\beta_2)}

\newcounter{example}
\renewcommand{\theexample}{\thesection.\arabic{example}}

\newcounter{remark}
\renewcommand{\theremark}{\thesection.\arabic{remark}}

\def\Ebb{\mathbb{E}}
\newlength{\noteWidth}
\long\def\notes#1{\ifinner
{\tiny #1}
\else
\marginpar{\parbox[t]{\noteWidth}{\raggedright\tiny #1}}
\fi\typeout{#1}}

 \def\notes#1{\typeout{read notes: #1}} 

\newcommand{\ie}{i.e.\@\xspace} 

\newcommand{\Real}{\ensuremath{\mathbb{R}}}

\def\Ebb{\mathbb{E}}

\def\exp{\mathop{\hbox{\rm exp}}}

\def\spose#1{\hbox to 0pt{#1\hss}}

\def\text #1{\hbox{\quad#1\quad}}

\def\nthinsp{\mskip -2   mu}

\def\superstar{^{\raise 0.5pt\hbox{$\nthinsp *$}}}
\def\SUPERSTAR{^{\raise 0.5pt\hbox{$*$}}}

\def\lamstarT {\lambda^{\raise 0.5pt\hbox{$\nthinsp *$}T}}

\def\Lscr{{\cal L}}

\def\Pscr{{\cal P}}
\def\Qscr{{\cal Q}}
\def\Sscr{{\cal S}}

\def\Wscr{{\cal W}}

\def\Nscr{{\cal N}}

\def\Zscr{{\cal Z}}

\def\Kscr{{\cal K}}

\def\non{\nonumber}

\let\forallnew\forall
\renewcommand{\forall}{\forallnew\ }
\let\forall\forallnew

		\def\bkE{{\rm I\kern-.17em E}}
		\def\bk1{{\rm 1\kern-.17em l}}
		\def\bkD{{\rm I\kern-.17em D}}
		\def\bkR{{\rm I\kern-.17em R}}
		\def\bkP{{\rm I\kern-.17em P}}
		\def\bkY{{\bf \kern-.17em Y}}
		\def\bkZ{{\bf \kern-.17em Z}}
		\def\bkC{{\bf  \kern-.17em C}}

{\begin{list}{}%
         {\setlength{\leftmargin}{#1}}%
         \item[]%
}
{\end{list}}

		\def\bsp{\begin{split}}
		\def\beq{\begin{eqnarray}}
		\def\bal{\begin{align*}}
		\def\bc{\begin{center}}
		\def\be{\begin{enumerate}}
		\def\bi{\begin{itemize}}
		\def\bs{\begin{small}}
		\def\bS{\begin{slide}}
		\def\ec{\end{center}}
		\def\ee{\end{enumerate}}
		\def\ei{\end{itemize}}
		\def\es{\end{small}}
		\def\eS{\end{slide}}
		\def\eeq{\end{eqnarray}}
		\def\eal{\end{align*}}
		\def\esp{\end{split}}
		\def\qed{ \vrule height7.5pt width7.5pt depth0pt}  

	\def\cp2problem#1#2#3#4{\fbox
		 {\begin{tabular*}{0.9\textwidth}
			{@{}l@{\extracolsep{\fill}}l@{\extracolsep{6pt}}l@{\extracolsep{\fill}}c@{}}
				#1 & & $#4 $ 
			\end{tabular*}}}

		\def\bkE{{\rm I\kern-.17em E}}
		\def\bk1{{\rm 1\kern-.17em l}}
		\def\bkD{{\rm I\kern-.17em D}}
		\def\bkR{{\rm I\kern-.17em R}}
		\def\bkP{{\rm I\kern-.17em P}}
		
		\def\bkZ{{\bf{Z}}}

\newcommand {\beeq}[1]{\begin{equation}\label{#1}}
\newcommand {\eeeq}{\end{equation}}
\newcommand {\bea}{\begin{eqnarray}}
\newcommand {\eea}{\end{eqnarray}}

\def\texitem#1{\par\smallskip\noindent\hangindent 25pt
               \hbox to 25pt {\hss #1 ~}\ignorespaces}

\def\bsp{\begin{split}}
		\def\beq{\begin{eqnarray}}
		\def\bal{\begin{align*}}
		\def\bc{\begin{center}}
		\def\be{\begin{enumerate}}
		\def\bi{\begin{itemize}}
		\def\bs{\begin{small}}
		\def\bS{\begin{slide}}
		\def\ec{\end{center}}
		\def\ee{\end{enumerate}}
		\def\ei{\end{itemize}}
		\def\es{\end{small}}
		\def\eS{\end{slide}}
		\def\eeq{\end{eqnarray}}
		\def\eal{\end{align*}}
		\def\esp{\end{split}}
		\def\qed{ \vrule height7.5pt width7.5pt depth0pt}  

\usepackage{amsmath, amssymb, xspace}
\usepackage{epsfig}
\usepackage{longtable}
\usepackage{color}
\usepackage{mathrsfs}
\usepackage{subfig}
\newenvironment{proof}[1][]{{\noindent \textit{ Proof}: }}{\hfill \qed \vspace{3pt}\\ }

\def\Nscr{{\cal N}}

\title{A Unified PAC-Bayesian Framework for Machine Unlearning via Information Risk Minimization}
\name{Sharu Theresa Jose, Osvaldo Simeone 
}
\address{Department of Engineering\\ King's College London\\London, WC2R 2LS}

\begin{document}

\maketitle

\begin{abstract}
Machine unlearning refers to mechanisms that can remove the influence of a subset of training data upon request from a  trained model without incurring the cost of re-training from scratch. This paper develops a unified PAC-Bayesian framework for machine unlearning that recovers the two recent design principles -- variational unlearning \cite{nguyen2020variational} and forgetting Lagrangian \cite{golatkar2020eternal}-- as information risk minimization problems \cite{zhang2006information}. Accordingly, both criteria can be interpreted as  PAC-Bayesian upper bounds on the test loss of the unlearned model that take the form of free energy metrics.
\end{abstract}
\section{Introduction}
\label{sec:intro}
AI tools are increasingly widespread and subject to privacy attacks and data misuse. Recent regulations, such as the European Union's General Data Protection Regulation (GDPR) and the California Consumer Privacy Act, has enshrined in law the right for individuals to 
withdraw consent to the use of their personal data for training machine learning models. The mere deletion of the requested data from the training data set does not serve the purpose, as information about the deleted data can still be retrieved from already trained machine learning models  \cite{carlini2019secret}. Thus, data deletion necessitates the machine learning model to \textit{unlearn} the contribution of the deleted data to the training process, such that the resulting model behaves as if it has never observed the data in the first place.

A straightforward approach to unlearn is to retrain the model from scratch by using only data remaining after deletion of the data to be unlearnt. However, this is computationally intensive and resource expensive. 
 \textit{Machine unlearning} refers to mechanisms that can remove the influence of a specific  subset of the training data on a trained machine learning model, without incurring the cost of retraining from scratch \cite{cao2015towards}, \cite{ginart2019making}. 
 
 Several machine unlearning approaches have been studied since the introduction of the concept in \cite{cao2015towards}, where the problem was studied in the context of statistical query learning. \cite{bourtoule2019machine} proposes an unlearning approach that partitions the data set into \textit{shards} that are used to train multiple models in \textit{isolation} and finally \textit{aggregrated}. This allows unlearning to be carried out by aggregating only the remaining shards, avoiding the need for retraining.
 
 Our work is motivated by two recently proposed machine unlearning mechanisms. The first proposes a design criterion, termed Evidence Upper BOund (EUBO), for \textit{variational unlearning} within a Bayesian setting \cite{nguyen2020variational}, while  the second optimizes over a ``scrubbing function'' by minimizing a  \textit{forgetting Lagrangian} criterion \cite{golatkar2020eternal}. Although prima facie these two approaches seem different, we demonstrate that the two design principles can be interpreted in a unified manner in the context of PAC-Bayesian theory \cite{mcallester1999pac}, \cite{germain2009pac}. PAC-Bayesian theory develops high-probability upper bounds on the population loss of a learning algorithm in terms of a free energy metric that includes the sum of a training loss and the Kullback-Leibler (KL) divergence between the learning algorithm and a data-independent \textit{prior} distribution \cite{germain2009pac, jose2021free}.
 
The main contributions of the paper are summarized as follows.
      We develop a unified PAC-Bayesian framework for machine unlearning that explains the  unlearning design principles in \cite{nguyen2020variational} and \cite{golatkar2020eternal}
      through the principle of information risk minimization (IRM) \cite{zhang2006information}. The PAC-Bayesian formulation makes use of the recent result in \cite{rivasplata2020pac} that accounts for data-dependent priors.
      We show that the design criteria -- EUBO and forgetting Lagrangian -- optimize PAC-Bayesian bounds  with appropriate choices of training loss  and data-dependent prior. Finally, the proposed framework motivates the design of \textit{amortized} variants of variational unlearning and forgetting Lagrangian-based mechanisms, which are also described.
 
\section{Learning and Unlearning Algorithms}
In this section, we start by defining the operation and performance criteria of learning and unlearning algorithms. These are described as stochastic mappings as in the standard PAC-Bayes framework.
\subsection{Learning Algorithm}
Let $D=(Z_1,\hdots,Z_n)$ denote a training data set
 of $n$ samples generated i.i.d. according to an unknown population distribution $P_Z \in \Pscr(\Zscr)$. A learning algorithm uses the data set $D$ to infer a model parameter $W$ belonging to a model class $\Wscr$. We define the learning algorithm as a stochastic mapping, $P_{W|D} \in \Pscr(\Wscr)$\footnote{We use $\Pscr(\cdot)$ to denote the space of all probability distributions on $`\cdot$'.}, from the input training set $D$ to the model class, $\Wscr$. The probabilistic mapping $P_{W|D}$ describes a distribution over all possible outcomes $W$ in the model class $\Wscr$.

Let $\ell:\Wscr \times \Zscr \rightarrow \Real_{+}$ denote a loss function. The goal of the learning algorithm is to find a model parameter $w \in \Wscr$ that minimizes the \textit{population loss},
\begin{align}
L(w)=\Ebb_{P_Z}[\ell(w,Z)], \label{eq:populationloss}
\end{align}which is the average loss of the model parameter $w$ incurred on a new test data point $Z \sim P_Z$.
 The population loss \eqref{eq:populationloss} is unknown to the learner, since the underlying  population distribution $P_Z$ is not available. Instead, the learner uses the \textit{empirical training loss} on the data set $D$, \ie,
\begin{align}
\Lh(w|D)=\frac{1}{n}\sum_{i=1}^n \ell(w,Z_i) \label{eq:trainingloss}
\end{align} as the training criterion. For a given training data set $D$, we define the \textit{generalization error}, $\Delta \Lscr(P_{W|D})$, of a learning mechanism $P_{W|D}$ as the average difference between the population loss \eqref{eq:populationloss} and the training loss \eqref{eq:trainingloss}, \ie, \begin{align}\Delta \Lscr(P_{W|D})=\Ebb_{P_{W|D}}[L(W)-\Lh(W|D)]. \label{eq:generalizationerror}\end{align} The generalization error \eqref{eq:generalizationerror} quantifies the extent to which the training loss \eqref{eq:trainingloss} can be reliably used as a  proxy measure for the unknown population loss.
  \subsection{Machine Unlearning}
  Consider a model $W_l \sim P_{W|D}$ learned using the data set $D$. When a  request is received to ``delete" a subset $D_e \subset D$ of $m$ samples, the learned model $W_l$ must be updated so as to ``unlearn" the information extracted from the data set $D_e$ by the learning process. We refer to data set $D_e$ as the \textit{unlearning data set}. Ideally, this could be done by re-training from scratch by using the remaining data, $D_r=D \setminus D_e$, \ie, by applying the stochastic mapping $P_{W|D_r}$. Given the large computational cost of re-training, \textit{machine unlearning} aims to remove the influence of the data $D_e$ on the learned model $W_l$ without incurring the full cost of re-training from scratch. 
Formally, we define an unlearning algorithm as follows \cite{sekhari2021remember}.
 \begin{definition}[Unlearning Algorithm]\label{def:unlearning}
    An unlearning algorithm $P_{W|W_l,T(D),D_e}$  is a stochastic  mechanism that maps the learned model parameter $W_l \sim P_{W|D}$, a statistic $T(D)$ of data set $D$, and the unlearning data set $D_e$ to the space of model parameters $\Wscr$.
  \end{definition} 
 The rationale for making the unlearned model $W \sim P_{W|W_l,T(D),D_e}$ depend on a statistic $T(D)$ of $D$ is to rule out training from scratch. In fact, if the statistic is $T(D) = D$, the unlearning algorithm can ignore $W_l$ and re-train from scratch, while more restrictive choices of $T(D)$ make this impossible.
 
 In order to ensure successful unlearning, one needs to impose that the distribution of the unlearned model $W$ be close to that obtained by training from scratch. For fixed data sets $D$ and $D_e$, the latter distribution is $P_{W|D_r}$,  while the former is given by the average $\Ebb_{P_{W_l|D}}[P_{W|W_l,T(D),D_e}]$ over the learning mechanism. Note that the expectation marginalizes over the learned models. This constraint can be formalized as follows.
\begin{definition}[$\epsilon$-certified unlearning]\label{def:certificate}
 An unlearning algorithm $P_{W|W_l,T(D),D_e}$  is said to satisfy $\epsilon$-certified unlearning for $\epsilon>0$ if
 \begin{align}
 D_{\KL}(\Ebb_{W_l \sim P_{W|D}}[P_{W|W_l,T(D),D_e}]||P_{W|D_r})\leq \epsilon, \label{eq:epsilon_unlearning}
 \end{align}
 where $D_{\KL}(P||Q)$ denotes the KL divergence between distributions $P$ and $Q$.
 \end{definition}

By the biconvexity of the KL divergence, it is easy to see that the unlearning certificate in \eqref{eq:epsilon_unlearning} is implied by the stronger condition that the inequality
  \begin{align}
 D_{\KL}(P_{W|W_l,T(D),D_e}||P_{W|D_r})\leq \epsilon \label{eq:epsilon_unlearning_1}
 \end{align} applies for all $W_l \in \Wscr$ in the support of $P_{W|D}$.

 
 \section{Preliminaries}
 In this section, we briefly review the classical  PAC-Bayesian framework, which underlies the proposed unified approach to machine unlearning. PAC Bayesian theory \cite{mcallester1999pac,mcallester2003pac} provides upper bounds on the average population loss, $\Ebb_{P_{W|D}}[L(W)]$, of a learning algorithm $P_{W|D}$ in terms of: $(a)$ the average training loss, $\Ebb_{P_{W|D}}[\Lh(W|D)]$, and $(b)$ the KL divergence between the distribution $P_{W|D}$ and an arbitrary data-independent ``prior" $Q_W$. The PAC-Bayesian bounds hold with high probability over random draws of the training data set $D$. There has been extensive study on various refinements to the original PAC-Bayesian bound of \cite{mcallester1999pac} (see \cite{guedj2019primer} for a review).  More  recently, PAC-Bayesian bounds have been extended to account for \textit{data-dependent} priors \cite{dziugaite2018data}, \cite{rivasplata2020pac}. 
 
 In this work, we make use of the general PAC-Bayesian bound derived in Theorem~2 of \cite{rivasplata2020pac} that allows for data-dependent priors. This turns out to be important for unlearning, since the prior will be used to account for the learning algorithm. The next lemma restates Theorem~2 in \cite{rivasplata2020pac} by using our notation and by adopting a conventional formulation in terms of uniform bounds over all posteriors $P_{W|D}$. A proof is provided for completeness in Appendix~\ref{app:PACBayesian_new}.
\begin{lemma}\label{lem:PACBayesian_new}
Let $Q_{W|D}$ denote a data-dependent prior. For any (measurable) function $A : \Zscr^n \times \Wscr \rightarrow \Real^2$ and convex function
$F : \Real^2 \rightarrow \Real$, let $f :\Zscr^n \times \Wscr \rightarrow \Real$ be the composition of $F$ and $A$, and let $\xi = \Ebb_{P^{\otimes n}_Z}\Ebb_{Q_{W|D}}[\exp(f(D,W)) ]$. Then, with probability at least $1-\delta$, with $\delta \in (0,1)$, over the random draw of data set $D \sim P^{\otimes n}_Z$ , the following inequality holds uniformly over all stochastic mappings $P_{W|D}$
\begin{align}
 &F(\Ebb_{P_{W|D}}[A(D,W)])\non \\& \leq D_{\KL}(P_{W|D}||Q_{W|D}) + \log(\xi/\delta) . \label{eq:PAC-Bayesianbound}
\end{align}
\end{lemma}

In the rest of the paper, we will use Lemma~\ref{lem:PACBayesian_new} by selecting function $A(D,W)$ to output a tuple including the population loss $L(W)$ and a training loss metric to be specialized for different unlearning methods. Furthermore, the convex function $F$ will be chosen to output the difference of its inputs, \ie, $F(a,b)=a-b$. With these choices, the PAC-Bayesian bound in \eqref{eq:PAC-Bayesianbound} will allow us to relate the empirical training metrics and the unknown population loss.

For reference, in the standard analysis of learning algorithms, the function $A(D,W)$ is selected to be the two-dimensional vector $[\beta L(W),\beta \Lh(W|D)]$. With this choice, the bound in \eqref{eq:PAC-Bayesianbound} can be re-written as an upper bound on the population loss that holds for all $P_{W|D}$:
\begin{align}
   & \Ebb_{P_{W|D}}[L(W)] \leq \mathcal{F}_{\mathrm{IRM}}+ \frac{1}{\beta}\log(\xi/\delta), \quad \mbox{where}, \label{eq:PAC-bayesianbound_1}\\
    &\mathcal{F}_{\mathrm{IRM}}=\Ebb_{P_{W|D}}[\Lh(W|D)]+\frac{1}{\beta}D_{\KL}(P_{W|D}||Q_{W|D}).\non
\end{align}
Important to our framework is the observation that the PAC-Bayesian bound  \eqref{eq:PAC-Bayesianbound}, and hence also \eqref{eq:PAC-bayesianbound_1}, hold uniformly over all choices of the learning algorithm $P_{W|D}$. As such, one can optimize the right-hand side of \eqref{eq:PAC-bayesianbound_1} ove the learning algorithm $P_{W|D}$ by considering the problem $\min_{P_{W|D}}\mathcal{F}_{\mathrm{IRM}}$. By minimizing an upper bound on the population loss, the learning criterion \eqref{eq:PAC-bayesianbound_1} facilitates generalization. This approach is known as \textit{Information Risk Minimization (IRM)} \cite{zhang2006information}, and it amounts to the minimization of a \textit{free energy criterion} \cite{jose2021free}. A free energy criterion is given by the sum of a training loss and of an information-theoretic regularization.

The PAC-Bayesian bound in \eqref{eq:PAC-Bayesianbound} contains a constant term $\xi$, bounding which ensures non-vacuous bounds on the generalization error. For data-independent priors, under suitable assumptions on the loss function, such as boundedness or sub-Gaussianity, the constant $\xi$ can be easily upper bounded. An upper bound on $\xi$ for a data-dependent prior has been recently obtained in \cite{rivasplata2020pac}. Since we will use \eqref{eq:PAC-Bayesianbound} to justify unlearning criteria via variants of the IRM problem, we will not be further concerned with bounding $\xi$.

\section{Variational Unlearning}
In this section, we study the Bayesian unlearning framework introduced in the recent work \cite{nguyen2020variational}. As we first review, this paper presents a new unlearning criterion, termed \textit{Evidence Upper BOund} (EUBO), that enables variational unlearning. To be consistent with Definition~\ref{def:unlearning}, we specifically describe here an amortized variational unlearning variant of the approach proposed in \cite{nguyen2020variational}. We then show that the resulting unlearning algorithm can be interpreted as IRM, which is obtained through a specific instantiation of the PAC-Bayesian bound \eqref{eq:PAC-Bayesianbound}.

\subsection{Amortized Variational Unlearning}\label{sec:variationalunlearning}
In order to meet the unlearning requirement \eqref{eq:epsilon_unlearning} for some $\epsilon>0$, the variational unlearning framework proposed in \cite{nguyen2020variational} finds a distribution in the model parameter space $\Wscr$ that is closest, in terms of KL divergence, to the distribution $P_{W|D_r}$ resulting from re-training on the remaining data $D_r$. Optimization is restricted to a given family of distributions.

The approach requires the variational optimization to be carried out separately for any given selection of data sets $D$ and $D_e$. Furthermore, it relies on access to the distribution $P_{W|D}$ and not solely on a trained model $W_l$. In contrast, an efficient unlearning mechanism conforming to Definition~\ref{def:unlearning} must define a conditional probability distribution $P_{W|W_l,T(D),D_e}$ that can be instantiated for any choice of learned model $W_l$, statistic $T(D)$ of the data, and unlearning data set $D_e$. To this end, in this section, we develop an amortized variant of variational unlearning \cite{nguyen2020variational} that enables optimization over an unlearning mechanism $P_{W|W_l,T(D),D_e}$. We refer to this approach as \textit{amortized variational unlearning} (AVU).

The proposed AVU framework constrains the unlearning mechanism $P_{W|W_l,T(D),D_e}$ to  belong to a family $\Qscr^{\AVU}$ of (parameterized) conditional distributions on $\Wscr$. AVU seeks to find the unlearning mechanism $P_{W|W_l,T(D),D_e}$ that solves the following problem
\begin{align}
 \min_{\substack{P_{W|W_l,T(D),D_e} \\ \in \Qscr^{\AVU}}} \Ebb_{P_{D,D_e} P_{W_l|D}}\Bigl[D_{\KL}(P_{W|W_l,T(D),D_e}||P_{W|D_r})\Bigr] \label{eq:variational_unlearning},
\end{align}where $P_{W_l|D}$ denote the distribution of the learned model $W_l \sim P_{W|D}$, and $P_{D,D_e}$ denote the probability distribution of the training data $D \sim P^{\otimes n}_{Z}$ and of the unlearning data set $D_e \sim P_{D_e|D}$. The conditional distribution $P_{D_e|D}$ describes a uniformly distributed stochastic selection of a subset $D_e$ of $m$ samples  from $D$. Problem \eqref{eq:variational_unlearning} aims at ensuring that the unlearning condition \eqref{eq:epsilon_unlearning_1} be satisfied on average over all training data set $D$ and unlearning data set $D_e$ for small value of $\epsilon>0$.

Following \cite{nguyen2020variational},  the optimization problem in \eqref{eq:variational_unlearning} can be equivalently formulated as 
\begin{align}
\min_{\substack{P_{W|W_l,T(D),D_e}\\ \in \Qscr^{\AVU}}}\hspace{-0.3cm} \Ebb_{P_{D,D_e} P_{W_l|D}}\Bigl[\EUBO(P_{W|W_l,T(D),D_e},P_{W|D})\Bigr] \label{eq:variational_unlearning-1},
\end{align}where the Evidence Upper BOund (EUBO) is defined as
\begin{align}
&\EUBO(P_{W|W_l,T(D),D_e},P_{W|D})\non \\=& \Ebb_{P_{W|W_l,T(D),D_e}}[\log P_{D_e|W}] \non \\& + D_{\KL}(P_{W|W_l,T(D),D_e}||P_{W|D}). \label{eq:EUBO}
\end{align}
The EUBO \eqref{eq:EUBO} comprises of two terms: $(i)$ the average positive log-likelihood of the unlearning data set $D_e$ obtained after unlearning; and $(ii)$ the deviation of the unlearning mechanism from the learning algorithm $P_{W|D}$. Intuitively, the first term should be small for effective unlearning, while the second is a regularization penalty that accounts for the residual epistemic uncertainty associated with the training algorithm.
\subsection{A PAC-Bayesian View of Variational Unlearning}
We now demonstrate that the optimization  \eqref{eq:EUBO} can be justified as an IRM obtained from the PAC-Bayesian bound in \eqref{eq:PAC-Bayesianbound}.
To instantiate the PAC-Bayesian bound in \eqref{eq:PAC-Bayesianbound} for unlearning, we note that the unlearning mechanism in Definition~\ref{def:unlearning} is a cascade of two operations: $(a)$ sample model parameter $W_l \sim P_{W|D}$ according to the learning mechanism; and then $(b)$ apply the unlearning mechanism $P_{W|W_l,T(D),D_e}$ on the learned model $W_l$. This process is subject to the random draw of data $D \sim P^{\otimes n}_Z$ and to the random selection of subset of data to be removed, $D_e \sim P_{D_e|D}$.  In line with this observation, we have the following PAC-Bayesian bound for the unlearning mechanism.
\begin{corollary}\label{thm:PACBayesian_variationalunlearning}
Let the data dependent prior be fixed as the learning mechanism $P_{W|D}$. With probability at least $1-\delta$, with $\delta \in (0,1)$, over the random draw of the data set $D \sim P^{\otimes n}_{Z}$ and the subset $D_e \subset D$ to be removed, the following inequality holds uniformly for all unlearning algorithms $P_{W|W_l,T(D),D_e}$:
\begin{align}
&\Ebb_{ P_{W_l|D}}\Ebb_{P_{W|W_l,T(D),D_e}}[-\Ebb_{P_Z}[\log P_{Z|W}]] \non \\\leq &\Ebb_{ P_{W_l|D}}\biggl[\frac{1}{m} \EUBO(P_{W|W_l,T(D),D_e},P_{W|D})\biggr]\non \\&+ \frac{1}{m}\log \frac{\bar{\xi}_{\AVU}}{\delta} ,\label{eq:PAC-Bayesian-unlearning-negativeloss}
\end{align}
where
$\bar{\xi}_{\AVU}=\Ebb_{P_{D,D_e} P_{W|D}}[\exp(m(-\Ebb_{P_Z}[\log P_{Z|W}]-(1/m)\log P_{D_e|W})].$
\end{corollary}
\begin{proof}
This result is obtained from Lemma~\ref{lem:PACBayesian_new}  by selecting $A(D,W)$ as the two-dimensional vector $[-m\Ebb_{P_{Z}}[\log P_{Z|W}],$ $ \log P_{D_e|W}]$ and $F(a,b)=a-b$. Details can be found in Appendix~\ref{app:PACBayesian_variationalunlearning}.
\end{proof}

The left-hand side in \eqref{eq:PAC-Bayesian-unlearning-negativeloss} is the average test log-loss obtained by the unlearnt model. Therefore, by \eqref{eq:PAC-Bayesian-unlearning-negativeloss}, the variational unlearning mechanism introduced in \cite{nguyen2020variational} can be interpreted as minimizing an upper bound on the test log-loss over the unlearning mechanism $P_{W|W_l,T(D),D_e}$ (assuming knowledge of $P_{W|D}$). By \eqref{eq:EUBO}, this minimization is of the form \eqref{eq:PAC-bayesianbound_1} assumed by IRM problems \cite{zhang2006information}. As $\delta \rightarrow 0$, the inequality in \eqref{eq:PAC-Bayesian-unlearning-negativeloss} holds almost surely, which justifies taking the average in \eqref{eq:EUBO} over the draws of $D$ and $D_e$. It follows that the proposed AVU \eqref{eq:EUBO} can be similarly interpreted in terms of the minimization of a PAC-Bayes upper bound on the average test log-loss, and hence in terms of an IRM problem.
\section{Forgetting Lagrangian-Based Unlearning}
In this section, we first review the unlearning  framework introduced in \cite{golatkar2020eternal}, the \textit{Forgetting Lagrangian}, and show that this can also be intrepreted as an IRM obtained as a specific instantiation of \eqref{eq:PAC-Bayesianbound}.

\subsection{Forgetting Lagrangian}
Reference \cite{golatkar2020eternal} considers a stochastic learning mechanism $P_{W|D}$ that trains the model parameter vector $W$ of a deep neural network (DNN) using data set $D$. The unlearning mechanism $P_{W|W_l,T(D),D_e}$ ignores the statistic $T(D)$ and  yields a \textit{stochastic scrubbing function} $P_{W|W_l,D_e}$ that ``scrubs off'' the influence of the unlearning data set $D_e$ on the learned model $W_l \sim P_{W|D}$. 

The scrubbing function $P_{W|W_l,D_e}$ is designed so as to optimize the \textit{Forgetting Lagrangian},
\begin{align}
&\mathcal{FL}(P_{W|W_l,D_e},\lambda)=\Ebb_{P_{W|W_l,D_e}}[\Lh(W|D_r)]\non \\& +\lambda D_{\KL}(\Ebb_{P_{W_l|D}}[P_{W|W_l,D_e}]||\Ebb_{P_{W_l|D_r}}[\tilde{P}_{W|W_l}]) \label{eq:forgettinglagrangian}
\end{align} where $\lambda>0$ denotes a Lagrangian multiplier, and $\tilde{P}_{W|W_l}$ is a  an arbitrary `reference' distribution that maps the model $W_l \sim P_{W|D_r}$, obtained by retraining on the data set $D_r$, to a ``noisy" version $W \in \Wscr$. The forgetting Lagrangian in \eqref{eq:forgettinglagrangian} thus aims at finding an unlearning mechanism that $(a)$ minimizes the average training loss $\Lh(w|D_r)$ on the remaining data $D_r$; while $(b)$ ensuring that the unlearning mechanism $P_{W|W_l,D_e}$ applied on the learned model $W_l \sim P_{W|D}$ is close, in terms of KL divergence, to the reference distribution $\tilde{P}_{W|W_l}$ applied on the model $W_l \sim P_{W|D_r}$ obtained after re-training from scratch. Thus, the KL divergence term in \eqref{eq:forgettinglagrangian} ensures a ``certificate of unlearning" with respect to the reference $\tilde{P}_{W|W_l}$ in the sense of Definition~\ref{def:certificate}. Moreover, the KL divergence term can be interpreted as an upper bound on the information about the unlearning data set $D_e$ that can be read out from observing the unlearned model $W \sim P_{W|W_l,D_e}$ \cite{golatkar2020eternal}.

As discussed in Section~\ref{sec:variationalunlearning}, designing the unlearning mechanism via the forgetting Lagrangian in \eqref{eq:forgettinglagrangian} requires the optimization to be performed for each selection of the learned model $W_l$ and the data sets $D$ and $D_e$. Furthermore, it depends directly on the distribution $P_{W|D}$. Following the discussion in Section~\ref{sec:variationalunlearning}, we could address this problem by considering an \textit{amortized forgetting Lagrangian} approach so as to optimize a conditional distribution $P_{W|W_l,D_e}$ that can be instantiated for any choice of learned model $W_l$, and unlearning data $D_e$. We do not pursue this here, since reference \cite{golatkar2020eternal} shows that an approximate solution $P_{W|W_l,D_e}$ to problem \eqref{eq:forgettinglagrangian} can be found that does not require a separate optimization for all $D$ and $D_e$.
\subsection{A PAC-Bayesian view of forgetting Lagrangian}
We now show that the forgetting Lagrangian \eqref{eq:forgettinglagrangian} follows from a specific instantiation of the PAC-Bayesian bound \eqref{eq:PAC-Bayesianbound} for unlearning mechanisms.

\begin{corollary}\label{thm:PACBayesian_forgettinglagrangian}
Let the data dependent prior be fixed as $\tilde{P}_{W|D,D_e}=\Ebb_{P_{W_l|D_r}}[\tilde{P}_{W|W_l}]$. Then, for all $\beta>0$, with probability at least $1-\delta$, with $\delta \in (0,1)$, over the random draw of the data set $D \sim P^{\otimes n}_{Z}$ and the subset $D_e \subset D$ to be removed, the following inequality holds uniformly for all $P_{W|W_l,D_e}$,
\begin{align}
&\Ebb_{P_{W_l|D}P_{W|W_l,D_e}}[L(W)] \non \\&\leq \Ebb_{P_{W_l|D}}[\mathcal{FL}(P_{W|W_l,D_e},\beta^{-1})]+\frac{1}{ \beta }\log \frac{\xi_{\mathcal{FL}}}{\delta_2} ,\label{eq:PAC-Bayesian-unlearning-forgettinglagrangian}
\end{align}
where
$\xi_{\mathcal{FL}}=\Ebb_{P_{D,D_e} \tilde{P}_{W|D,D_e}}[\exp(\beta(L(W)-\Lh(W|D_r))].$
\end{corollary}
\begin{proof}
The proof follows in the same steps as the proof of Corollary~\ref{thm:PACBayesian_variationalunlearning} with $A(D,W)=[\beta L(W), \beta \Lh(W|D_r)]$. Details in Appendix~\ref{app:forgettinglagrangian}.
\end{proof}

The left-hand side of \eqref{eq:PAC-Bayesian-unlearning-forgettinglagrangian} is the average test loss, and hence the forgetting Lagrangian framework introduced in \cite{golatkar2020eternal} can be again interpreted as minimizing an upper bound on the average test loss.
\section{Conclusion}
The paper presents a unified PAC-Bayesian framework for the design of machine unlearning algorithms. We show that two unlearning design criteria studied in literature -- EUBO for variational unlearning \cite{nguyen2020variational} and Forgetting Lagrangian \cite{golatkar2020eternal} can be interpreted as IRM obtained via specific instantiation of the proposed PAC-Bayesian framework. 
\appendix
\section{Proof of Lemma~3.1}\label{app:PACBayesian_new}
The  PAC-Bayesian bound in \eqref{eq:PAC-Bayesianbound} is obtained by first using a Markov inequality, and then applying change of measure as detailed next. The  Markov inequality for a non-negative random variable $Y$ states that with probability at least $1-\delta$, with $\delta \in (0,1)$, we have $Y \leq \Ebb[Y]/\delta$. Precisely, the following inequality holds,
$$\mathrm{Pr}(Y \leq \Ebb[Y]/\delta) \geq 1-\delta. $$
We specialize the above Markov inequality to our setting by taking $Y=\Ebb_{Q_{W|D}}[\exp(f(D,W))]$. Note that $Y$ is a function of the random variable $D$, and that
$
    \Ebb_{P^{\otimes n}_Z}[Y]=\xi.
    $ Markov's inequality then gives that
\begin{align}
&\mathrm{Pr}_D\biggl( \Ebb_{Q_{W|D}}[\exp(f(D,W)] \leq \frac{\xi}{\delta}\biggr )\geq 1-\delta .\label{eq:1}
\end{align} Applying change of measure then results in the following inequality
\begin{align}
    \mathrm{Pr}_D\biggl(\forall P_{W|D}, \hspace{0.01cm} &\Ebb_{P_{W|D}}\biggl[\exp\biggl(\hspace{-0.1cm} f(D,W)-\log \frac{P_{W|D}(W|D)}{Q_{W|D}(W|D)}\hspace{-0.05cm}\biggr) \biggr]\non \\&\leq \frac{\xi}{\delta}\biggr)
    \geq 1-\delta.
\end{align} Using Jensen's inequality to take expectation inside the exponential term, and subsequently applying log on both sides of the inequality then results in
\begin{align}
   \mathrm{Pr}_D\biggl(\forall P_{W|D}, \hspace{0.2cm} &\Ebb_{P_{W|D}}[f(D,W)]-D_{\KL}(P_{W|D}||Q_{W|D}) \non \\& \leq \log \frac{\xi}{\delta}\biggr)\geq 1-\delta. 
\end{align} Finally, noting that $f(D,W)=F(A(D,W))$ where $F$ is convex, and applying Jensen's inequality again results in the PAC-Bayesian bound in \eqref{eq:PAC-Bayesianbound}.
\section{Proof of Corollary~4.1}\label{app:PACBayesian_variationalunlearning}
The required bound follows by instantiating the general PAC-Bayesian bound in Lemma~\ref{lem:PACBayesian_new} for unlearning. As such, the unlearning PAC-Bayesian bound depends on the cascade operation of learning a model $W_l \sim P_{W|D}$, and subsequent unlearning using $P_{W|W_l,T(D),D_e}$. This process is subject to the random draw of $D \sim P^{\otimes n}_Z$, and to the random selection of the subset $D_e \subset D$.  Consequently, we consider the prior in Lemma~\ref{lem:PACBayesian_new} as $Q_{W|D,D_e}$, depending on both data sets $D$ and $D_e$.

Lemma~\ref{lem:PACBayesian_new} then gives that with probability at least $1-\delta$ over the random draw of data set $D$, and that of the unlearning data set $D_e$, the following inequality holds uniformly over all distributions $P_{W|D,D_e}$,
\begin{align}
    F(\Ebb_{P_{W|D,D_e}}[A(D,W)]) &-D_{\KL}(P_{W|D,D_e}||Q_{W|D,D_e}) \non \\&\leq \log \frac{\xi}{\delta}.\label{eq:imp}
\end{align}
In particular, \eqref{eq:imp} holds for all learning mechanisms $P_{W|D}$ and unlearning mechanisms $P_{W|W_l,T(D),D_e}$ such that $P_{W|D,D_e}=\Ebb_{ P_{W_l|D}}[P_{W|W_l,T(D),D_e}]$ is the marginal of the joint distribution $P_{W_l|D} \otimes P_{W|W_l,T(D),D_e}$.

To get to \eqref{eq:PAC-Bayesian-unlearning-negativeloss}, we consider the PAC-Bayesian bound \eqref{eq:imp} for a fixed learning algorithm $P_{W|D}$. Further, we take $Q_{W|D,D_e}=P_{W|D}$, $A(D,W)=[-m\Ebb_{P_Z}[\log P_{Z|W}],$ $ \log P_{D_e|W}]$ and $F(a,b)=a-b$. Noting that $P_{W|D,D_e}=\Ebb_{ P_{W_l|D}}[P_{W|W_l,T(D),D_e}]$, we  use the biconvexity of KL divergence to upper bound
\begin{align*}
    &D_{\KL}(P_{W|D,D_e}||P_{W|D}) \non \\&\leq \Ebb_{P_{W_l|D}}[D_{\KL}[P_{W|W_l,T(D),D_e}||P_{W|D}].
\end{align*}Using all these in \eqref{eq:imp} yields the required bound in \eqref{eq:PAC-Bayesian-unlearning-negativeloss}.
\section{Proof of Corollary~5.1}\label{app:forgettinglagrangian}
The proof follows the same line as the proof of Corollary~\ref{thm:PACBayesian_variationalunlearning} in Appendix~\ref{app:PACBayesian_variationalunlearning}. To get to \eqref{eq:PAC-Bayesian-unlearning-forgettinglagrangian}, we use \eqref{eq:imp} with $P_{W|D,D_e}=\Ebb_{ P_{W_l|D}}[P_{W|W_l,T(D),D_e}]$, $Q_{W|D,D_e}=\Ebb_{ P_{W_l|D_r}}[\tilde{P}_{W|W_l}]$ and $A(D,W)=[\beta L(W),$ $ \beta \Lh(W|D_r)].$
\bibliographystyle{IEEEbib}
\bibliography{ref}

\end{document}